\documentclass{article}
\PassOptionsToPackage{numbers, compress}{natbib}

\usepackage[main, final]{neurips_2026}

\usepackage[utf8]{inputenc}
\usepackage[T1]{fontenc}
\usepackage{url}
\usepackage{microtype}
\usepackage{graphicx}
\usepackage{subcaption}
\usepackage{caption}
\usepackage{fontawesome}
\usepackage{booktabs}
\usepackage{array}
\usepackage{tabularx}
\usepackage{multirow}
\usepackage[table]{xcolor}
\usepackage{calc}
\usepackage{amsmath}
\usepackage{amssymb}
\usepackage{amsfonts}
\usepackage{mathtools}
\usepackage{amsthm}
\usepackage{bm}
\usepackage{float}
\usepackage{wrapfig}
\usepackage{needspace}
\usepackage{algorithm}
\usepackage{algcompatible}
\usepackage{enumitem}
\usepackage[most]{tcolorbox}
\usepackage{tikz}
\usetikzlibrary{shadows.blur}
\usetikzlibrary{decorations.pathreplacing}
\usepackage{hyperref}
\usepackage[capitalize,noabbrev]{cleveref}

\algrenewcommand\alglinenumber[1]{%
  \scriptsize
  \ifnum#1<10\relax
    \hphantom{0}%
  \fi
  #1:%
}

\theoremstyle{plain}
\newtheorem{theorem}{Theorem}[section]

\newtheorem{corollary}[theorem]{Corollary}
\theoremstyle{definition}

\theoremstyle{remark}

\newcommand{\ForComment}[2]{
    \renewcommand{\algorithmicdo}{\textbf{do} \hfill \textit{#2}}
    \FOR{#1}
    \renewcommand{\algorithmicdo}{\textbf{do}}
}

\newcommand{\ourmethod}{PISA}

\newcommand{\github}{\raisebox{-0pt}{\faGithub}}

\definecolor{rebuttalblue}{RGB}{0,65,190}
\newcommand{\rebuttalsetup}{%
  \color{rebuttalblue}%
  \hypersetup{citecolor=rebuttalblue,linkcolor=rebuttalblue,urlcolor=rebuttalblue}%
  \captionsetup{font+={color=rebuttalblue},labelfont+={color=rebuttalblue}}%
}

\definecolor{highlightgray}{gray}{0.9}
\definecolor{iccvblue}{rgb}{0.21,0.49,0.74}
\definecolor{citeblue}{HTML}{1F5FAD}
\definecolor{lightyellow}{RGB}{240,230,149}
\definecolor{tableblue}{rgb}{0.9,0.945,0.988}
\definecolor{tablebrown}{rgb}{0.949,0.933,0.9255}
\definecolor{darkred}{rgb}{0.8,0.004,0}
\definecolor{darkgreen}{rgb}{0,0.6,0}
\definecolor{linkpurple}{HTML}{d8137e}

\hypersetup{
  colorlinks=true,
  citecolor=citeblue,
  linkcolor=darkred,
  urlcolor=linkpurple
}

\usepackage{booktabs,tabularx,multirow,colortbl,xcolor,array}
\newcolumntype{Y}{>{\centering\arraybackslash}X}

\title{Piecewise Sparse Attention Is Wiser for Efficient Diffusion Transformers}
\title{\ourmethod{}: Piecewise Sparse Attention Is Wiser for Efficient\\Diffusion Transformers}

\vspace{-5pt}
\author{%
    \vspace{-25pt}\\
    \textbf{Haopeng Li}$^1$ \quad
    \textbf{Shitong Shao}$^1$ \quad
    \textbf{Wenliang Zhong}$^1$ \quad
    \textbf{Zikai Zhou}$^1$ \quad
    \textbf{Lichen Bai}$^1$ \\
    \vspace{-0.5em}\\
    \textbf{Hui Xiong}$^{1,2}$ \quad \quad
    \textbf{Zeke Xie}$^{1,2,}$\thanks{Corresponding author.} \\
    \vspace{-0.5em}\\
    $^1$The Hong Kong University of Science and Technology (Guangzhou) \\
    $^2$The Hong Kong University of Science and Technology \\
    \vspace{-0.5em}\\
    \github\,\,Code: \href{https://github.com/xie-lab-ml/piecewise-sparse-attention}{https://github.com/xie-lab-ml/piecewise-sparse-attention}
    \vspace{-0.25em}\\
}

\begin{document}

\setcounter{footnote}{1}
\maketitle

\vspace{-2.5em}
\begin{figure}[h]
    \centering
    \captionsetup{font={small}, labelfont=bf, skip=3pt}
    \input{fig/title}
    \vspace{-12pt}
    \caption{\textbf{\ourmethod{} accelerates diverse generation tasks.} \textbf{(a, b)} Wan2.1-14B image-to-video generation. \textbf{(c)} FLUX.1 image generation. \ourmethod{} achieves better quality than SpargeAttn at higher sparsity. \textbf{(d)} Wan2.1-1.3B attention patterns: \ourmethod{} more closely approximates dense attention than block sparse attention at similar cost.}
    \label{fig:title}\label{fig:attn_score}
\end{figure}

\begin{abstract}
Diffusion Transformers are fundamental for visual generation, but their efficiency is bottlenecked by the quadratic complexity of self-attention. While block-sparse attention accelerates computation by attending only to critical key-value blocks, it suffers from quality degradation at high sparsity due to discarding substantial context. In this work, we discover that attention scores of non-critical blocks exhibit distributional stability, allowing their contributions to be accurately and efficiently recovered rather than dropped. Motivated by this insight, we propose \textbf{\ourmethod{}}, a training-free \underline{Pi}ecewise \underline{S}parse \underline{A}ttention that covers the full attention span via lightweight compensation for skipped non-critical blocks. Unlike the conventional ``keep-or-drop'' paradigm that directly drops non-critical blocks, \ourmethod{} establishes the ``exact-or-approximate'' paradigm: it preserves exact computation for critical blocks while approximating the remainder through \textit{block-wise Taylor expansion} around block centroids. This design allows \ourmethod{} to serve as a faithful proxy to full attention, achieving 1.91$\times$ and 2.57$\times$ speedups on Wan2.1 and Hunyuan-Video, and 1.2--1.6$\times$ acceleration on FLUX without compromising quality.
\end{abstract}
\section{Introduction}
\label{sec:intro}

Diffusion Transformers (DiTs)~\cite{peebles2023scalable} have demonstrated impressive performance and scalability in high-fidelity image and video generation across diverse visual tasks~\cite{wan2025,kong2024hunyuanvideo,chen2025sana,flux2024,shao2026core,liu2026alignment,bai2026mainecoon,liu2026sol,shao2026liveditor}. However, as demand for higher resolutions and longer video durations grows, the sequence length of input tokens increases dramatically. Consequently, the quadratic complexity of self-attention~\cite{vaswani2017attention} becomes a significant bottleneck to efficient inference in large-scale DiTs.

To address the computational bottleneck, especially in high-resolution image and video generation, recent research has leveraged the inherent sparsity in DiTs to enable sparse attention. Early works~\cite{zhang2025fast,xisparse} capitalized on the spatiotemporal redundancy of video diffusion transformers to introduce static, training-free sparse attention patterns. To improve adaptability, other methods~\cite{zhang2025spargeattn,xuxattention,xia2025training,yang2025sparse} propose computing sparse patterns dynamically at runtime. Moving beyond training-free approaches, recent works explore trainable sparse attention~\cite{zhang2025vsa,li2025radial}, adopt linear attention for efficient generation~\cite{xie2024sana,chen2025sana,liu2024linfusion}, or combine sparse and linear attention~\cite{zhang2025sla,xi2026videodeltanet}.

However, existing methods still face inherent limitations: (1) Hard truncation: Sparse attention directly discards key-value pairs, leading to performance drops at high sparsity and inefficiency on shorter sequences (e.g., 4K tokens). (2) Incompatibility with pre-trained weights: Linear and hybrid attention fundamentally alter the attention distribution of pre-trained models, precluding the direct reuse of weights and necessitating expensive retraining. These limitations underscore the need for \textit{a unified mechanism that enhances efficiency without sacrificing quality or requiring retraining}.

To this end, we propose \textbf{\ourmethod{}}, a novel training-free sparse attention method that accelerates DiTs while maintaining high accuracy through piecewise computation. Unlike standard sparse attention, which computes only critical blocks and discards the rest, \ourmethod{} treats attention as a piecewise process: (1) Exact computation for sparse key-value blocks to preserve critical information; (2) Approximation for the remaining blocks using block-wise Taylor expansion to cover the massive amount of non-critical information. 

Specifically, we propose a hybrid-order approximation strategy that uses block-wise zero-order expansion and global first-order approximation to efficiently improve accuracy. This significantly improves approximation fidelity to full attention with negligible computational overhead over standard sparse attention, as illustrated in Figure~\ref{fig:attn_score}(d). These dual computational pathways are fused into a single online softmax process, achieving a state-of-the-art trade-off between efficiency and accuracy.

Extensive experiments demonstrate the superiority of \ourmethod{}. It accelerates Wan2.1-14B~\cite{wan2025} and Hunyuan-Video~\cite{kong2024hunyuanvideo} by 1.91$\times$ and 2.57$\times$, respectively, while preserving state-of-the-art quality. Even for image generation tasks with lower inherent sparsity, \ourmethod{} outperforms existing methods in both efficiency and quality. 

Our contributions are summarized as follows:

1. We propose a novel training-free piecewise sparse attention that enables full attention span while substantially reducing computational cost. Through a unified exact-or-approximate computation, it resolves the critical dilemma between accuracy and efficiency.

2. We develop a hybrid-order approximation scheme that boosts accuracy at negligible cost. We also derive a covariance-aware routing strategy from error analysis to reduce approximation divergence.

3. Experiments demonstrate that our method achieves state-of-the-art quality and efficiency across diverse tasks, setting a new sparse attention paradigm for efficient visual generation.

\section{Related Work}
\label{sec:related}

\paragraph{Block Sparse Attention.}
To address the quadratic complexity of standard attention~\cite{shao2026efficient}, sparse attention~\cite{zhang2025spargeattn,zhang2025vsa,xisparse,yang2025sparse,li2025radial,wu2025vmoba,ding2026draftattention2} limits computation to a subset of critical key-value blocks via static priors or dynamic routing. However, current methods simply discard the unselected blocks, which inevitably exacerbates output error and degrades performance at a high sparsity ratio. In contrast, \textbf{\ourmethod{}} guarantees a strictly lower error bound than standard sparse attention by approximating the unselected blocks instead of dropping them, enabling superior performance.

\paragraph{Bidirectional Linear Attention.}
Linear attention~\cite{katharopoulos2020transformers,huang2026mdn} achieves linear complexity by replacing the exponential kernel with feature mappings. In the vision domain, existing bidirectional methods primarily focus on modifying these feature mappings~\cite{liu2024linfusion,mengpolaformer,han2023flatten}, yet their core formulation remains consistent with the canonical framework. Notably, Taylor-based approaches~\cite{arora2024simple} typically expand around zero for the entire sequence. Although kernel tricks~\cite{gelada2025scaling} can yield high-order approximations, they incur severe dimension explosion and fail to preserve the pre-trained attention distribution. Consequently, these methods require computationally expensive retraining.

\paragraph{Native Hybrid Attention.}
Recognizing the limitations of pure sparse or linear approaches, recent works have explored hybrid architectures. Methods like SLA~\cite{zhang2025sla} and NHA~\cite{du2025native} attempt to combine sparse (or sliding window) attention with linear attention. However, SLA directly sums the outputs of different branches, which disrupts the intrinsic normalization of attention weights, while NHA adds linear-RNN memory slots; both require training to adapt pre-trained models. In contrast, \textbf{\ourmethod{}} applies block-wise expansion to both the normalization numerator and denominator. By natively mixing exact and approximate terms under a unified softmax framework instead of simply adding outputs, our method preserves the intrinsic distribution and enables superior training-free performance.

\section{Methodology}
\label{sec:method}
\vspace{-2pt}
\subsection{Preliminaries}
\vspace{-2pt}
Given an input sequence $\boldsymbol{X} \in \mathbb{R}^{L \times d}$, where $L$ is the length and $d$ is the feature dimension, the query, key, and value $\boldsymbol{Q}, \boldsymbol{K}, \boldsymbol{V} \in \mathbb{R}^{L \times d}$ are derived from $\boldsymbol{X}$ via learnable linear projections. The output $\boldsymbol{O} \in \mathbb{R}^{L \times d}$ of attention is (omit scale factor):
\begin{equation}
    \boldsymbol{O} = \text{Softmax}\left( {\boldsymbol{Q}\boldsymbol{K}^\top } \right)\boldsymbol{V}.
\end{equation}
FlashAttention~\cite{dao2022flashattention} introduces \textit{online softmax} that avoids materializing attention scores in high bandwidth memory (HBM), significantly reducing memory access overhead. 
To further mitigate the computation complexity, sparse attention restrict computation to a subset of critical key-value blocks:
\begin{equation}
    \boldsymbol{O} = \text{Softmax}\left({\boldsymbol{Q}\boldsymbol{K}^\top } + \boldsymbol{M}\right)\boldsymbol{V}, \quad \boldsymbol{M} \in \{0,-\infty\}^{L \times L}.
\end{equation}
Here, $\boldsymbol{M}$ denotes a mask where entries of $-\infty$ indicate the corresponding key-value pairs are ignored. 

\begin{figure}[t]
    \centering
    \includegraphics[width=\linewidth]{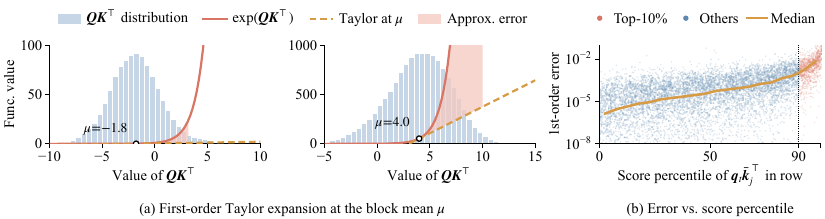}
    \captionsetup{skip=3pt}
    \caption{\textbf{Error of first-order Taylor expansion on pre-softmax attention scores ($\boldsymbol{Q}\boldsymbol{K}^\top$) in Wan2.1-1.3B.} \textbf{(a)} Expansion at the block mean $\mu$ (note the $10\times$ $y$-range on the right). \textbf{(b)} Error vs.\ within-row score percentile (all layers/heads): the top-10\% blocks carry 75\% of the error.}
    \label{fig:dist_attn_score}
\end{figure}

However, directly discarding blocks inevitably causes the output to deviate sharply from the original attention distribution. This limitation motivated us to design \textit{a fast and accurate sparse attention mechanism capable of accelerating inference in a training-free manner}.

\vspace{-6pt}
\paragraph{Key Insights.}
To identify a superior alternative to the ``keep-or-drop'' strategy, we analyzed the statistical properties of pre-trained models, yielding two key insights:

\noindent\textit{(1) Pre-softmax scores of uncritical blocks exhibit a symmetric distribution centered around zero or negative values}, rendering them highly amenable to approximation via mean-centered Taylor expansion, as illustrated in Figure~\ref{fig:dist_attn_score}.

\noindent\textit{(2) Normalization Consistency.}
Existing hybrid methods rely on an additive strategy ($\boldsymbol{O}_{sparse} + \boldsymbol{O}_{linear}$) that violates the intrinsic weighted sum rule. To ensure training-free compatibility, the approximation must be integrated \textit{internally} into the softmax numerator and denominator.

\subsection{Piecewise Sparse Attention}
Based on these insights, we propose Piecewise Sparse Attention (\ourmethod{}), which unifies exact sparse computation with block-wise approximation directly within the online softmax. The following subsection outline our framework.
\vspace{-6pt}
\paragraph{Formulation.} 
We partition the query, key, value and output matrices ($\boldsymbol{Q}, \boldsymbol{K}, \boldsymbol{V}, \boldsymbol{O}$) into blocks of size $B$. 
To streamline notation, we express the query and output vectors using a global index. 
Let $\boldsymbol{q}_t := \boldsymbol{Q}_{iB+m} \in \mathbb{R}^{1 \times d}$ (where $t=iB+m$) denote the query vector at global index $t$, which is the $m$-th row vector belongs to the $i$-th $\boldsymbol{Q}$ block. 
For the key and value sides, we preserve the block-internal structure: let $\boldsymbol{k}_{j,n}$ and $\boldsymbol{v}_{j,n}$ denote the $n$-th row vectors of the $j$-th key and value blocks, respectively (where $1 \le n \le B$).

For an arbitrary query $\boldsymbol{q}_t$, we partition the key-value block indices into a sparse selected set $\mathcal{S}_i$ (computed exactly) and a long-tail unselected set $\mathcal{U}_i$. 
Instead of discarding the unselected blocks as in standard sparse attention, we approximate the contribution from each block $j \in \mathcal{U}_i$ via \textit{First-Order Taylor expansion} centered at the block centroid $\alpha_{t,j} := \exp(\boldsymbol{q}_t \bar{\boldsymbol{k}}_j^\top)$, where $\bar{\boldsymbol{k}}_j := \frac{1}{B}\sum_{n=1}^B \boldsymbol{k}_{j,n}$. 
The output vector $\boldsymbol{o}_{t} \in \mathbb{R}^{1 \times d}$ is derived by normalizing the weighted value aggregation (omit scale):
\begin{equation}
    \boldsymbol{o}_{t} := \frac{ \mathcal{N}_{t} }{ \mathcal{D}_{t} }, \quad \text{where} \quad \notag
    \mathcal{D}_{t} := \underbrace{\sum_{j \in \mathcal{S}_i} \sum_{n=1}^B \exp({\boldsymbol{q}_{t} \boldsymbol{k}_{j,n}^\top})}_{\text{Sparse Exact}} + \underbrace{\sum_{j \in \mathcal{U}_i} B \cdot \exp({\boldsymbol{q}_{t} \bar{\boldsymbol{k}}_j^\top})}_{\text{Block-wise Approx}}, \label{eq:block-1st-d}
\end{equation}
\begin{equation}
\resizebox{0.98\linewidth}{!}{$\displaystyle
    \mathcal{N}_{t} := \underbrace{\sum_{j \in \mathcal{S}_i} \sum_{n=1}^B \exp({\boldsymbol{q}_{t} \boldsymbol{k}_{j,n}^\top)} \boldsymbol{v}_{j,n}}_{\text{Sparse Exact}}
    + \underbrace{\sum_{j \in \mathcal{U}_i} \exp({\boldsymbol{q}_{t} \bar{\boldsymbol{k}}_j^\top)} \sum_{n=1}^B \boldsymbol{v}_{j,n}}_{\text{Block-wise Zeroth-order Approx}}
    + \underbrace{\boldsymbol{q}_{t} \sum_{j \in \mathcal{U}_i} \exp(\boldsymbol{q}_{t} \bar{\boldsymbol{k}}_j^\top)\sum_{n=1}^B (\boldsymbol{k}_{j,n} - \bar{\boldsymbol{k}}_j)^\top \boldsymbol{v}_{j,n} }_{\text{Block-wise First-order Approx}}.
$}
\end{equation}
Note that the first-order term in the denominator $\mathcal{D}_t$ cancels out because $\sum_{n=1}^B (\boldsymbol{k}_{j,n} - \bar{\boldsymbol{k}}_j)^\top = 0$. 

Since queries within the same block share identical block masking patterns, they can naturally generalize to the full block via matrix operations. %

\vspace{-6pt}
\paragraph{Practical Challenge.}
The \textit{``Sparse Exact''} term corresponds to standard block sparse attention, which can be computed efficiently. Similarly, \textit{``Block-wise Zeroth-order Approx''} term can be efficiently computed via matrix multiplication (GEMM) by grouping pre-computed $\bar{\boldsymbol{k}}_j$ and $\sum_{n=1}^B \boldsymbol{v}_{j,n}$ into sub-blocks.
However, for \textit{``Block-wise First-order Approx''} term, although it theoretically exhibits linear complexity, its practical implementation is severely bottlenecked by memory access. Computing this first-order term necessitates handling distinct matrices $\sum_{n=1}^B (\boldsymbol{k}_{j,n} - \bar{\boldsymbol{k}}_j)^\top \boldsymbol{v}_{j,n} \in \mathbb{R}^{d \times d}$ for each block $j \in \mathcal{U}_i$, weighted by block-specific scalar. 
Whether these matrices are pre-computed and loaded from HBM or computed on-the-fly in SRAM, the process results in a \textit{memory-bound} operation with \textit{low arithmetic intensity}, rendering the theoretical speedup unattainable in practice. 
To resolve this conflict between theoretical complexity and hardware efficiency, we propose two solutions.

\begin{figure*}[t!]
    \centering
    \includegraphics[width=\linewidth]{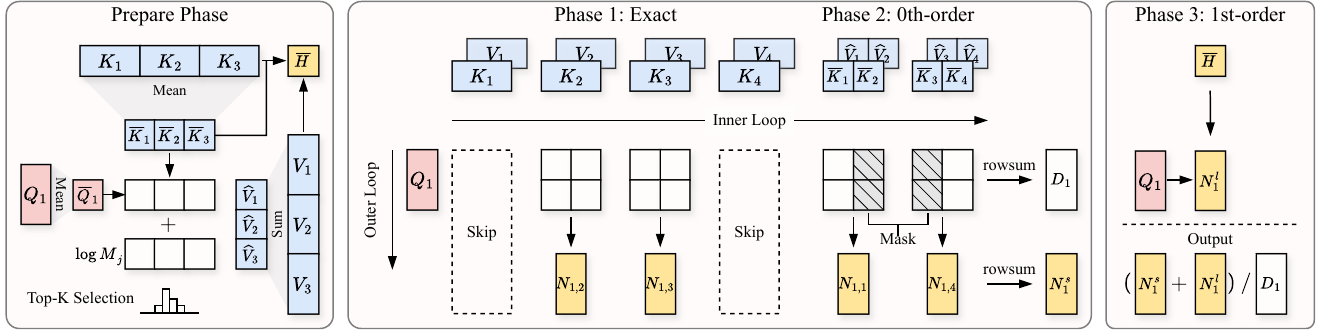}
    \caption{\textbf{The algorithm pipeline of \ourmethod{}.} \textit{Prepare Phase:} we pre-compute block-wise means of queries and keys ($\overline{Q},\overline{K}$), block-wise sums of values ($\hat{V}$) and the global statistic $\bar{H}$ in a single pass; a block-wise Top-K selection (Eq.~\eqref{eq:topk}) identifies critical blocks. \textit{Fused kernel:} selected blocks (e.g., 2, 3) are computed exactly (\textit{Phase 1}), unselected blocks (e.g., 1, 4) use the block-wise zeroth-order expansion with selected columns masked (\textit{Phase 2}), and \textit{Phase 3} applies $\bar{H}$ once to inject the global first-order term, avoiding memory-bound streaming.}
  \label{fig:ours_pipeline}
  \vspace{-6pt}
\end{figure*}

\subsection{Hybrid Approximation}

\paragraph{Global First-Order Correction.}
We propose a hybrid-order approximation in which the first-order term is formulated globally across all unselected blocks, rather than block-wise. This allows all blocks to share a unified scale factor $\beta_t$, thereby eliminating the need for weighted summation. Let $\bar{\boldsymbol{k}}_t$ denote the global centroid of keys across all blocks in $\mathcal{U}_i$. The 1st-order term can be rewritten as:
\begin{equation}
    \beta_t \boldsymbol{q}_t \sum_{j\in \mathcal{U}_i} \sum_{n=1}^B (\boldsymbol{k}_{j,n} - \bar{\boldsymbol{k}}_t)^\top \boldsymbol{v}_{j,n}.
\end{equation}
Determining the expansion coefficient $\beta_t$ presents a critical challenge. While a standard first-order Taylor expansion at the global centroid $\bar{\boldsymbol{k}}_t$ would suggest $\beta_t = \exp(\boldsymbol{q}_t \bar{\boldsymbol{k}}_{t}^\top)$, Jensen's Inequality indicates that this approach severely underestimates the slope of the global first-order term, rendering the correction ineffective. To address this, we define $\beta_t$ using the mean of $\exp(\boldsymbol{q}_t \bar{\boldsymbol{k}}_j^\top)$ effectively employing an average slope to estimate the magnitude of the first-order correction for all unselected blocks. Let $\boldsymbol{H}_j := \sum_{n=1}^B (\boldsymbol{k}_{j,n} - \bar{\boldsymbol{k}}_j)^\top \boldsymbol{v}_{j,n} \in \mathbb{R}^{d \times d}$. The 1st-order term can be rewritten as:
\begin{equation}
    \frac{1}{|\mathcal{U}_i|} \Big( \sum_{j \in \mathcal{U}_i} \exp(\boldsymbol{q}_t \bar{\boldsymbol{k}}_j^\top) \Big) \Big( \boldsymbol{q}_t \sum_{j\in \mathcal{U}_i} \boldsymbol{H}_j \Big).
\end{equation}
Furthermore, by associating the normalization factor $1/|\mathcal{U}_i|$ with the summation of $\boldsymbol{H}_j$, we essentially compute the mean of $\boldsymbol{H}_j$ over the unselected blocks. Observing that the set of unselected blocks $\mathcal{U}_i$ typically constitutes the vast majority of the total blocks, we approximate the query-dependent mean over $\mathcal{U}_i$ using a query-independent global statistic $\bar{\boldsymbol{H}} = \frac{1}{N}\sum_{j=1}^N \boldsymbol{H}_j$, where $N$ is number of blocks. This global statistic can be precomputed via a single pass. Consequently, we arrive at a computationally efficient formulation:
\begin{align}\label{eq:final}
    \mathcal{N}_{t} :=& \underbrace{\sum_{j \in \mathcal{S}_i} \sum_{n=1}^B \exp({\boldsymbol{q}_{t} \boldsymbol{k}_{j,n}^\top)} \boldsymbol{v}_{j,n}}_{\text{Exact Sparse Term}}
    + \underbrace{\sum_{j \in \mathcal{U}_i} \alpha_{t,j} \Big( \sum_{n=1}^B \boldsymbol{v}_{j,n} \Big)}_{\text{Block-wise 0th-order}}
    + \underbrace{\boldsymbol{q}_t \bar{\boldsymbol{H}} \sum_{j \in \mathcal{U}_i} \alpha_{t,j} }_{\text{Global 1st-order}}.
\end{align}

\paragraph{Error Analysis.} We provide a briefly error analysis below and the full proof is in Appendix~\ref{appendix:error_bound}.

\begin{tcolorbox}[
    colback=gray!5,
    colframe=black!25,
    boxrule=0.5pt,
    arc=2pt,
    left=6pt,
    right=6pt,
    top=6pt,
    bottom=6pt
]

\begin{theorem}[Total Error of PISA]
\label{thm:main_error}
Following Eq.~\eqref{eq:block-1st-d}--\eqref{eq:final}, let
$\boldsymbol o_t^\star$ be the dense attention output and
$\widehat{\boldsymbol o}_t$ be the PISA output. We use the scale-omitted notation of
Sec.~\ref{sec:method}. Assume $\|\boldsymbol q_t\|_2\le C_q$ and
$\|\boldsymbol v_{j,n}\|_2\le C_v$.

Let $\mathcal D_t^\star$ and $\widehat{\mathcal D}_t$ denote the dense and PISA
denominators, respectively, and let $\tau_t$ be the exact dense attention mass of the
unselected blocks. Define
\[
\eta_t:=\max_{j\in\mathcal U_i,n}
|\boldsymbol q_t(\boldsymbol k_{j,n}-\bar{\boldsymbol k}_j)^\top|,
\quad
\rho_t:=\frac{\tau_t}{\mathcal D_t^\star},
\quad
\zeta_t:=1-\frac{\widehat{\mathcal D}_t}{\mathcal D_t^\star},
\quad
M_t:=\max_{j\in\mathcal U_i}\|\boldsymbol H_j-\bar{\boldsymbol H}\|_2 .
\]
Then $0\le \zeta_t<1$, and
\[
\|\widehat{\boldsymbol o}_t-\boldsymbol o_t^\star\|_2
\le
\frac{\rho_t}{1-\zeta_t}
\left(
C_v e^{\eta_t}\eta_t^2
+
\frac{C_qM_t}{B}
\right).
\]
\end{theorem}
\end{tcolorbox}

The first term in the bound controls the second-order Taylor remainder and the
associated denominator perturbation, while the second term controls the replacement
of block-wise first-order matrices $\boldsymbol H_j$ by the global statistic
$\bar{\boldsymbol H}$. The factor $1/(1-\zeta_t)$ accounts for the change in softmax
normalization.

\begin{algorithm}[t]
\footnotesize
\renewcommand{\baselinestretch}{1.25}
\selectfont
\caption{\ourmethod{} Forward Pass}
\label{alg:span_attention}

\begin{algorithmic}[1]
\renewcommand{\algorithmicrequire}{\textbf{Input:}}
\renewcommand{\algorithmicensure}{\textbf{Output:}}

\REQUIRE $\boldsymbol{Q}, \boldsymbol{K}, \boldsymbol{V} \in \mathbb{R}^{L \times d}$, block size $B$ and $C$, sparsity $r$

\STATE Compute block means $\bar{\boldsymbol{Q}}, \bar{\boldsymbol{K}}$, block sum $\hat{\boldsymbol{V}}$ and global $\bar{\boldsymbol{H}}$; Partition $\bar{\boldsymbol{K}}, \hat{\boldsymbol{V}}$ into groups $\{\bar{\boldsymbol{K}}_{[g]}, \bar{\boldsymbol{V}}_{[g]}\}$ of size $C$.
\STATE Compute block indices $\mathcal{S}_i$ via Eq.~\eqref{eq:topk} for each query block $i$.

\FOR{$i \leftarrow 1$ \textbf{to} $N = \lceil L/B \rceil$}
    \STATE On-chip: Initialize $\boldsymbol{O}_i \in \mathbb{R}^{B \times d} \leftarrow \mathbf{0}$, \; $\boldsymbol{\ell}_i \in \mathbb{R}^{B} \leftarrow \mathbf{0}$, \; $\boldsymbol{\ell}^{\text{tail}}_i \in \mathbb{R}^{B} \leftarrow \mathbf{0}$, \; $\boldsymbol{m}_i \in \mathbb{R}^{B} \leftarrow -\infty$.
    
    \STATE Load $\boldsymbol{Q}_i \in \mathbb{R}^{B \times d}$ from HBM to on-chip SRAM.

    \ForComment{$j \in \mathcal{S}_i$}{Phase 1}
        \STATE Load $\boldsymbol{K}_j$, $\boldsymbol{V}_j \in \mathbb{R}^{B \times d}$ from HBM to on-chip SRAM.
        
        \STATE On-chip: $\boldsymbol{S}_{ij} = \boldsymbol{Q}_i \boldsymbol{K}_j^\top \in \mathbb{R}^{B \times B}$
        
        \STATE On-chip: $\boldsymbol{m}_i^{\text{new}} = \max(\boldsymbol{m}_i, \text{rowmax(}\boldsymbol{S}_{ij})) \in \mathbb{R}^{B}, \; \boldsymbol{\beta} = \exp(\boldsymbol{m}_i - \boldsymbol{m}_i^{\text{new}}) \in \mathbb{R}^{B}.$
        
         \STATE On-chip: $ \boldsymbol{P}_{ij} = \exp(\boldsymbol{S}_{ij}-\boldsymbol{m}_i^{\text{new}}) \in \mathbb{R}^{B \times B}$
         
        \STATE On-chip: $\boldsymbol{\ell}_i \leftarrow \boldsymbol{\beta} \cdot \boldsymbol{\ell}_i + \operatorname{rowsum}(\boldsymbol{P}_{ij})$
        
        \STATE On-chip: $\boldsymbol{O}_i \leftarrow \operatorname{diag}(\boldsymbol{\beta}) \boldsymbol{O}_i + \boldsymbol{P}_{ij} \boldsymbol{V}_j, \; \boldsymbol{m}_i \leftarrow \boldsymbol{m}_i^{\text{new}}$
    \ENDFOR

    \ForComment{$g \leftarrow 1$ \textbf{to} $\lceil N/C \rceil$}{Phase 2}
        \STATE Load $\bar{\boldsymbol{K}}_{[g]}$, $\bar{\boldsymbol{V}}_{[g]} \in \mathbb{R}^{C \times d}$ from HBM to on-chip SRAM.
        
        \STATE On-chip: $\boldsymbol{S}_{ig} = \boldsymbol{Q}_i \bar{\boldsymbol{K}}_{[g]}^\top \in \mathbb{R}^{B \times C}, \; \boldsymbol{S}_{ig}[\{j \mid j \in \mathcal{S}_i\}] \leftarrow -\infty$. \hfill \textit{\textcolor{gray}{\scriptsize Column-wise masking}}

        \STATE On-chip: $\boldsymbol{m}_i^{\text{new}} = \max(\boldsymbol{m}_i, \text{rowmax(}\boldsymbol{S}_{ig})) \in \mathbb{R}^{B}, \; \boldsymbol{\beta} = \exp(\boldsymbol{m}_i - \boldsymbol{m}_i^{\text{new}}) \in \mathbb{R}^{B}.$
        
        \STATE On-chip: $\boldsymbol{P}_{ig} = \exp(\boldsymbol{S}_{ig} - \boldsymbol{m}_i^{\text{new}}) \in \mathbb{R}^{B \times C}$
        
        \STATE On-chip: $\boldsymbol{\ell}_i \leftarrow \boldsymbol{\beta} \cdot \boldsymbol{\ell}_i + \operatorname{rowsum}(\boldsymbol{P}_{ig}), \; \boldsymbol{\ell}^{\text{tail}}_i \leftarrow \boldsymbol{\beta} \cdot \boldsymbol{\ell}^{\text{tail}}_i + \operatorname{rowsum}(\boldsymbol{P}_{ig})$.
        
        \STATE On-chip: $\boldsymbol{O}_i \leftarrow \operatorname{diag}(\boldsymbol{\beta}) \boldsymbol{O}_i + \boldsymbol{P}_{ig} \bar{\boldsymbol{V}}_{[g]}, \; \boldsymbol{m}_i \leftarrow \boldsymbol{m}_i^{\text{new}}$
    \ENDFOR

    \STATE On-chip: $\boldsymbol{R}_i = \operatorname{diag} (\boldsymbol{\ell}^{\text{tail}}_i) (\boldsymbol{Q}_i \bar{\boldsymbol{H}}) / L \in \mathbb{R}^{B \times d}$ \hfill {\textit{Phase 3}}
    \STATE On-chip: $\boldsymbol{O}_i \leftarrow \operatorname{diag}(\boldsymbol{\ell}_i^{-1}) (\boldsymbol{O}_i + \boldsymbol{R}_i)$
\ENDFOR
\ENSURE $\boldsymbol{O} = \{\boldsymbol{O}_i\}_{i=1}^M \in \mathbb{R}^{L \times d}$.
\end{algorithmic}
\end{algorithm}

\paragraph{Covariance-Aware Block Selection.}
Practically, to guarantee small approximation error it suffices to ensure either the tail fraction $\rho_t$ is small, or the per-block heterogeneity $M$ is small. 
Based on the Theorem~\ref{thm:main_error}, we propose a \textit{Covariance-Aware Top-k Strategy} that incorporates the norm of the block covariance matrix as a prior for importance routing. 

Let $M_j := \|\boldsymbol{H}_j - \bar{\boldsymbol{H}}\|_2$, the selection score for block $j$ with respect to query $\boldsymbol{q}_t$ is defined as:
\begin{equation}\label{eq:topk}
    \text{Score}_{t,j} = \text{Softmax}\left( \frac{\boldsymbol{q}_t \bar{\boldsymbol{k}}_j^\top}{\sqrt{d}} + \log\left(M_j + \epsilon\right) \right),
\end{equation}
where $\epsilon > 0$ is a small constant for numerical stability. The derivation of Eq.~\eqref{eq:topk} is further elaborated on in the Appendix~\ref{appendix:topk}. This strategy ensures that blocks with either high semantic relevance are preserved in the exact set $\mathcal{S}_i$, while the smoother blocks are delegated to the Taylor expansion.

\subsection{Hardware-aware Kernel Implementation}
\label{sec:implementation}

We design a fused kernel which efficiently interleaves exact computation with approximate scanning in online softmax. 
Concretely, the kernel proceeds in four phases:
\begin{itemize}[leftmargin=*,itemsep=5pt,topsep=4pt,parsep=0pt]
    \item \textbf{Prepare Phase:} block centroids, block-wise value sums and the global statistic $\bar{\boldsymbol{H}}$ are pre-computed in a single pass, and the critical blocks $\mathcal{S}_i$ of each query block are selected via Eq.~\eqref{eq:topk}.
    \item \textbf{Phase 1:} the selected blocks $j \in \mathcal{S}_i$ are loaded into SRAM for exact sparse attention, updating the running maximum and normalizer of the online softmax.
    \item \textbf{Phase 2:} aggregated block vectors are scanned densely in groups with coalesced memory access, while a dynamic mask excludes selected blocks, enabling high throughput for the tail mass.
    \item \textbf{Phase 3:} the global first-order term is injected once per query block before the final normalization, recovering first-order information at negligible cost.
\end{itemize}
Phases 1--3 form a single fused kernel; the pipeline is illustrated in Figure~\ref{fig:ours_pipeline} and Algorithm~\ref{alg:span_attention}.

\newlength{\pisasavedfloatsep}
\newlength{\pisasavedtextfloatsep}
\setlength{\pisasavedfloatsep}{\floatsep}
\setlength{\pisasavedtextfloatsep}{\textfloatsep}

\section{Experiments}
\label{sec:exp}
\subsection{Experimental Setup}\label{sec:exp_steup}
\paragraph{Models.} We evaluate \ourmethod{} on both video and image generation tasks. We employ Wan-2.1~\cite{wan2025} and Hunyuan-Video to produce videos at 480p and 720p resolutions. For image generation, we utilize SD 3.5~\cite{esser2024scaling} and FLUX.1-dev~\cite{flux2024} to generate images at 1024$\times $1024 (default) and 2048$\times $2048 resolutions.

\paragraph{Benchmarks.} We employ VBench~\cite{huang2023vbench} to evaluate both video quality and consistency. For image generation, we utilize FID~\cite{heusel2017gans} alongside a suite of human preference metrics, including ImageReward (IR)~\cite{xu2023imagereward}, HPSv2~\cite{wu2023human}, and MPS~\cite{MPS}. Furthermore, we compute SSIM, PSNR and LPIPS~\cite{zhang2018perceptual} to quantify the similarity between \ourmethod{} and full attention. Regarding efficiency, we note that theoretical FLOPs often fail to reflect real-world speeds. Therefore, to ensure a fair comparison of efficiency, we focus exclusively on end-to-end latency and the speedup ratio relative to full attention.

\begin{figure*}[t!]
  \centering
  \begin{subfigure}[b]{0.64\linewidth}
    \centering
    \includegraphics[width=\linewidth]{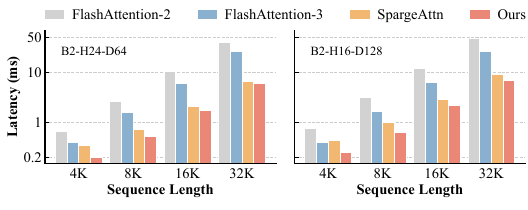}
    \caption{Latency Analysis}
    \label{fig:efficiency_latency}
  \end{subfigure}
  \hfill
  \begin{subfigure}[b]{0.34\linewidth}
    \centering
    \includegraphics[width=\linewidth]{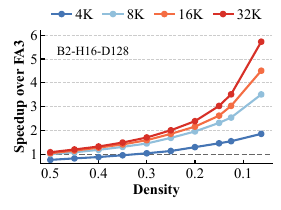}
    \caption{Speedup Analysis}
    \label{fig:efficiency_speedup}
  \end{subfigure}
  \caption{\textbf{Kernel efficiency profile.} 
  (\textit{a}) Latency comparison across sequence lengths at 12.5\% density (87.5\% sparsity) under two configurations with notation B-H-D (\textit{batch\_size}, \textit{num\_heads}, \textit{head\_dim}). 
  (\textbf{b}) Relative speedup against FlashAttention-3 (FA3) across varying densities and sequence lengths under the B2-H16-D128 configuration. The dashed line indicates the baseline performance.}
  \label{fig:main_efficiency}
\end{figure*}

\paragraph{Implementation Details.} We define sparsity as the proportion of blocks computed using approximation, analogous to the ratio of skipped blocks in standard sparse attention. Our kernel implementation uses a block size of 64$\times$64. Following prior works~\cite{li2025radial,yang2025sparse}, we adopt a warmup strategy where early layers and inference steps retain dense. Detailed configurations are provided in the Appendix~\ref{appendix:detail}.

\subsection{Kernel Efficiency Evaluation}

To evaluate the efficiency of our method, we benchmark against state-of-the-art implementations. Specifically, we adopt FlashAttention~\cite{dao2023flashattention2, shah2024flashattention} as the standard for full attention, and SpargeAttn~\cite{zhang2025spargeattn} as the baseline for block sparse attention. All results are profiled on NVIDIA H800 GPUs.

\paragraph{Efficiency vs. Sequence Length.} As illustrated in Figure~\ref{fig:efficiency_latency}, \ourmethod{} at a density of 12.5\% (sparsity 87.5\%) consistently outperforms FlashAttention-3 and SpargeAttn. Notably, even at shorter sequence (e.g., 4K), our method maintains a speed advantage over FlashAttention-3.

\begin{wrapfigure}{R}{0.47\textwidth}
    \vspace{-16pt}
    \centering
    {\includegraphics[width=\linewidth]{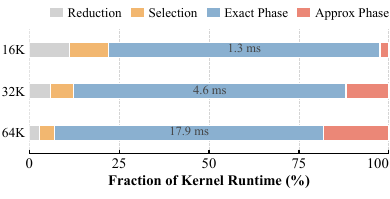}}
    \caption{\textbf{Latency breakdown of \ourmethod{}} at 12.5\% density. Exact attention (latency annotated) dominates the runtime at all lengths.}
    \vspace{-21.5pt}
    \label{fig:kernel_profile}
\end{wrapfigure}
\paragraph{Efficiency vs. Density.} Figure~\ref{fig:efficiency_speedup} illustrates the speedup of \ourmethod{} relative to FlashAttention-3 across varying densities and sequence lengths. \ourmethod{} is faster whenever the density is below 50\% for sequences of 8K and longer, and below 30\% for 4K; at 12.5\% density the speedup reaches 3.5$\times$ at 32K.

\paragraph{Runtime Breakdown.} We profile the runtime of distinct phases within our kernel. As illustrated in Figure~\ref{fig:kernel_profile}, the approximation phase incurs minimal overhead, confirming that our method improves accuracy without sacrificing runtime efficiency.

\begin{figure*}[t]
    \centering
    \input{fig/exp3}
    \vspace{-16pt}
    \caption{Video generation samples of different attention mechanisms. \textit{Left}: Hunyuan-Video 13B. \textit{Right}: Wan2.1-14B. All sparse attention are configured under a similar sparsity level ($\sim$87.5\%).}
    \label{fig:video_samples}
\end{figure*}

\newcolumntype{C}{>{\centering\arraybackslash}X}

\begin{table*}[t!]
\centering
\footnotesize
\setlength{\tabcolsep}{4pt}
\renewcommand{\arraystretch}{0.9}
\setlength{\aboverulesep}{0.8pt}
\setlength{\belowrulesep}{1pt}
\caption{Comparison of different sparse attention on video generation with warmup strategy. $\dagger$: %
using the official-provided configuration for reproducing the 87.13\% sparsity reported in SVG2 paper.}
\label{tab:main_video}
\begin{tabularx}{\textwidth}{lCCCCCCCCC}
\toprule
\multirow{2}{*}{\textbf{Method}} & \multirow{2}{*}{\textbf{Sparsity}} & \multicolumn{3}{c}{\textbf{VBench}$(\%)\uparrow$} & \multicolumn{3}{c}{\textbf{Similarity}} & \multicolumn{2}{c}{\textbf{Efficiency}} \\
\cmidrule(lr){3-5}\cmidrule(lr){6-8}\cmidrule(lr){9-10}
 & & \textbf{\scriptsize S.C.} & \textbf{\scriptsize I.Q.} & \textbf{\scriptsize A.Q.} & \textbf{\scriptsize SSIM}$\uparrow$ & \textbf{\scriptsize PSNR}$\uparrow$ & \textbf{\scriptsize LPIPS}$\downarrow$ & \textbf{\scriptsize Latency}$\downarrow$ & \textbf{\scriptsize Speedup}$\uparrow$ \\
\midrule

\multicolumn{10}{c}{\textit{\textbf{Wan2.1-1.3B} (Text-to-Video, 480P)}} \\
\midrule
  Dense & 0.00\% & 94.96 & 66.09 & 60.79 & -- & -- & -- & \hspace{\widthof{00}}98 s & 1.00$\times$ \\

  Sparge & 87.5\% & 93.76 & 64.89 & 58.66 & 0.761 & 20.75 & 0.138 & \hspace{\widthof{00}}51 s & 1.92$\times$ \\

  SVG & 87.5\% & 93.45 & 64.01 & 58.83 & 0.806 & 22.14 & 0.121 & \hspace{\widthof{00}}55 s & 1.78$\times$ \\
  
  SVG2$^\dagger$ & 87.1\% & 93.79 & 64.51 & 59.09 & 0.809 & 22.85 & 0.104 & \hspace{\widthof{00}}62 s & 1.58$\times$ \\
  
  \cellcolor{gray!12.5}Ours & \cellcolor{gray!12.5}87.5\% & \cellcolor{gray!12.5}94.58 & \cellcolor{gray!12.5}65.94 & \cellcolor{gray!12.5}60.03 & \cellcolor{gray!12.5}0.800 & \cellcolor{gray!12.5}22.62 & \cellcolor{gray!12.5}0.111 & 
  \cellcolor{gray!12.5}\hspace{\widthof{00}}48 s & \cellcolor{gray!12.5}2.04$\times$ \\

\midrule

\multicolumn{10}{c}{\textit{\textbf{Wan2.1-14B} (Text-to-Video, 720P)}} \\
\midrule
  Dense & 0.00\% & 95.98 & 67.71 & 63.08 & -- & -- & -- & 1564 s & 1.00$\times$ \\

  Sparge & 87.5\% & 95.69 & 67.11 & 62.72 & 0.766 & 21.47 & 0.144 & \hspace{\widthof{0}}844 s & 1.85$\times$ \\

  SVG & 87.5\% & 95.08 & 66.25 & 61.16 & 0.777 & 22.04 & 0.130 & \hspace{\widthof{0}}856 s & 1.83$\times$ \\

  SVG2$^\dagger$ & 87.1\% & 95.39 & 66.97 & 61.92 & 0.796 & 22.92 & 0.121 & \hspace{\widthof{0}}882 s & 1.77$\times$ \\
  
  \cellcolor{gray!12.5}Ours & \cellcolor{gray!12.5}87.5\% & \cellcolor{gray!12.5}95.80 & \cellcolor{gray!12.5}67.88 & \cellcolor{gray!12.5}63.38 & \cellcolor{gray!12.5}0.787 & \cellcolor{gray!12.5}22.69 & \cellcolor{gray!12.5}0.124 & \cellcolor{gray!12.5}\hspace{\widthof{0}}818 s &\cellcolor{gray!12.5}1.91$\times$ \\

\midrule

\multicolumn{10}{c}{\textit{\textbf{Hunyuan-13B} (Text-to-Video, 720P)}} \\
\midrule
  Dense & 0.00\% & 95.60 & 67.65 & 61.55 & -- & -- & -- & 1651 s & 1.00$\times$ \\

  Sparge & 87.5\% & 95.38 & 67.37 & 61.35 & 0.837 & 24.85 & 0.107 & \hspace{\widthof{0}}658 s & 2.51$\times$ \\

  SVG & 87.5\% & 94.50 & 62.99 & 56.34 & 0.833 & 26.04 & 0.106 & \hspace{\widthof{0}}661 s & 2.49$\times$ \\

  SVG2$^\dagger$ & 87.1\% & 94.88 & 63.10 & 58.58 & 0.848 & 26.40 & 0.109 & \hspace{\widthof{0}}649 s & 2.54$\times$ \\
  
  \cellcolor{gray!12.5}Ours & \cellcolor{gray!12.5}87.5\% & \cellcolor{gray!12.5}95.47 & \cellcolor{gray!12.5}68.16 & \cellcolor{gray!12.5}61.85 & \cellcolor{gray!12.5}0.840 & \cellcolor{gray!12.5}26.17 & \cellcolor{gray!12.5}0.106 & \cellcolor{gray!12.5}\hspace{\widthof{0}}641 s &\cellcolor{gray!12.5}2.57$\times$ \\

\bottomrule
\end{tabularx}
\end{table*}

\subsection{Visual Generation Evaluation}

\paragraph{Video Generation.} We evaluate our \ourmethod{} against recent sparse attention works (e.g., SpargeAttn~\cite{zhang2025spargeattn}, SVG~\cite{xisparse} and SVG2~\cite{yang2025sparse}) on the Wan2.1~\cite{wan2025} and Hunyuan-Video~\cite{kong2024hunyuanvideo}. 
As demonstrated in Table~\ref{tab:main_video}, our method achieves state-of-the-art performance on the VBench benchmark, significantly outperforming existing approaches while delivering substantial speedup advantages. Remarkably, it even surpasses full attention on certain generation quality metrics. 
In contrast, although SVG2 achieves high similarity scores, it suffers from severe degradation in overall video quality and consistency, exhibiting noticeable blurring and flickering between frames. Similarly, SpargeAttn struggles with the loss of fine-grained details, whereas \ourmethod{} exhibits visual quality consistent with full attention, as shown in Figure~\ref{fig:video_samples}.

\begin{table*}[!t]
\centering
\scriptsize
\setlength{\tabcolsep}{4pt}
\renewcommand{\arraystretch}{0.9}
\setlength{\aboverulesep}{0.8pt}
\setlength{\belowrulesep}{1pt}
\caption{Comparison of different sparse attention on text-to-image generation with warmup strategy.}
\label{tab:main_img}
\begin{tabularx}{\textwidth}{l l c *{7}{>{\centering\arraybackslash}X} c}
\toprule
\multirow{2}{*}{\textbf{Model}} & \multirow{2}{*}{\textbf{Method}} & \multirow{2}{*}{\textbf{Sparsity}} & \multicolumn{4}{c}{\textbf{Quality}} & \multicolumn{3}{c}{\textbf{Similarity}} & \multicolumn{1}{c}{\textbf{Efficiency}} \\
\cmidrule(lr){4-7}\cmidrule(lr){8-10}\cmidrule(lr){11-11}
 & & & \textbf{FID}$\downarrow$ & \textbf{IR}$\uparrow$ & \textbf{MPS}$\uparrow$ & \textbf{HPSv2}$\uparrow$ & \textbf{SSIM}$\uparrow$ & \textbf{PSNR}$\uparrow$ & \textbf{LPIPS}$\downarrow$ & \textbf{Latency}$\downarrow$ \\
\midrule
\multirow{4}{*}{\shortstack[l]{\textbf{SD 3.5-M}\\\textit{(1024p)}}}
 & Dense-Attn & \hspace{\widthof{0}}0\% & 15.48 & 0.912 & -- & 30.13 & -- & -- & -- & 3.03 s \\
 & MInference & 70\% & 17.30 & 0.678 & 39.87\% & 27.03 & 0.690 & 16.44 & 0.191 & 2.88 s \\
 & SpargeAttn & 70\% & 19.52 & 0.465 & 32.32\% & 26.30 & 0.627 & 14.60 & 0.269 & 2.69 s \\
 & \cellcolor{gray!12.5}Ours & \cellcolor{gray!12.5}70\% & \cellcolor{gray!12.5}14.78 & \cellcolor{gray!12.5}0.883 & \cellcolor{gray!12.5}49.71\% & \cellcolor{gray!12.5}30.40 & \cellcolor{gray!12.5}0.790 & \cellcolor{gray!12.5}19.55 & \cellcolor{gray!12.5}0.158 & \cellcolor{gray!12.5}2.15 s \\
\midrule
\multirow{4}{*}{\shortstack[l]{\textbf{FLUX.1}\\\textit{(1024p)}}}
 & Dense-Attn & \hspace{\widthof{0}}0\% & 16.35 & 0.965 & -- & 31.31 & -- & -- & -- & 8.32 s \\
 & MInference & 85\% & 20.02 & 0.860 & 40.01\% & 29.13 & 0.572 & 14.29 & 0.282 & 7.60 s \\
 & SpargeAttn & 85\% & 22.13 & 0.822 & 37.39\% & 27.66 & 0.498 & 11.08 & 0.298 & 7.47 s \\
 & \cellcolor{gray!12.5}Ours & \cellcolor{gray!12.5}85\% & \cellcolor{gray!12.5}15.91 & \cellcolor{gray!12.5}1.001 & \cellcolor{gray!12.5}51.98\% & \cellcolor{gray!12.5}31.69 & \cellcolor{gray!12.5}0.682 & \cellcolor{gray!12.5}17.07 & \cellcolor{gray!12.5}0.241 & \cellcolor{gray!12.5}6.87 s \\
\bottomrule
\end{tabularx}
\end{table*}

\paragraph{Image Generation.} 
We compare our method against SpargeAttn~\cite{zhang2025spargeattn} and MInference~\cite{jiang2024minference} (bidirectional version) on Stable Diffusion 3.5 (SD3.5)~\cite{esser2024scaling} and FLUX.1-dev~\cite{flux2024}, with results shown in Table~\ref{tab:main_img} (2048$\times$2048 results are in Table~\ref{tab:main_img_2048}). At similar sparsity levels, our method significantly outperforms other methods in FID and other human preference benchmarks, while achieving greater speedups. In terms of similarity, our method consistently surpasses other methods across different models, demonstrating that our piecewise attention mechanism preserves critical structural and semantic information. Visual results are presented in Figure~\ref{fig:title} and Appendix~\ref{appendix:samples}.

\begin{figure}[!t]
    \centering
    \begin{minipage}[t]{0.44\linewidth}
        \vspace{0pt}
        \centering
        \scriptsize
        \setlength{\tabcolsep}{2.5pt}
        \renewcommand{\arraystretch}{0.9}
        \setlength{\aboverulesep}{0.8pt}
        \setlength{\belowrulesep}{1pt}
        \captionof{table}{Ablations of approximation strategy. $\dagger$: with covariance-aware selection.}
        \label{tab:ablation1}
\begin{tabularx}{\linewidth}[t]{lYYYY}
\toprule
\textbf{Method} & \textbf{SSIM}$\uparrow$ & \textbf{PSNR}$\uparrow$ & \textbf{LPIPS}$\downarrow$ & \textbf{Spd.}$\uparrow$ \\
\midrule
\multicolumn{5}{c}{\textit{\textbf{FLUX.1-dev}}} \\
\midrule
\textit{0th} & 0.643 & 16.10 & 0.274 & 1.24$\times$ \\
\textit{1st} & 0.679 & 17.04 & 0.246 & 0.96$\times$ \\
\textit{hyd} & 0.677 & 17.01 & 0.248 & 1.22$\times$ \\
\cellcolor{gray!12.5}\textit{hyd}\smash{$^\dagger$} & \cellcolor{gray!12.5}0.682 & \cellcolor{gray!12.5}17.09 & \cellcolor{gray!12.5}0.241 & \cellcolor{gray!12.5}1.22$\times$ \\
\midrule
\multicolumn{5}{c}{\textit{\textbf{Wan2.1-14B}}} \\
\midrule
\textit{0th} & 0.772 & 21.68 & 0.136 & 1.93$\times$ \\
\textit{1st} & 0.788 & 22.84 & 0.122 & 0.91$\times$ \\
\textit{hyd} & 0.784 & 22.41 & 0.127 & 1.92$\times$ \\
\cellcolor{gray!12.5}\textit{hyd}\smash{$^\dagger$} & \cellcolor{gray!12.5}0.787 & \cellcolor{gray!12.5}22.69 & \cellcolor{gray!12.5}0.124 & \cellcolor{gray!12.5}1.91$\times$ \\
\bottomrule
\end{tabularx}
    \end{minipage}\hfill
    \begin{minipage}[t]{0.53\linewidth}
        \vspace{0pt}
        \centering
        \includegraphics[width=\linewidth]{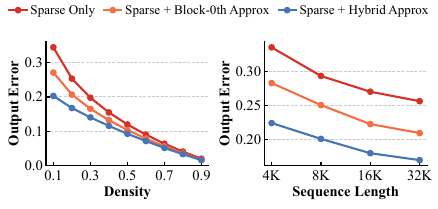}
        \vspace{-16.5pt}%
        \captionof{figure}{Quantitative ablation of approximation errors. Relative error w.r.t.\ dense attention.}
        \label{fig:abla_error}
    \end{minipage}
\end{figure}

\subsection{Ablation Study}

\paragraph{Approximation Errors.}
We examine the normalized $L_1$ error ratio of attention on Wan2.1-1.3B. As shown in Figure~\ref{fig:abla_error}, the zeroth-order variant of \ourmethod{} consistently achieves lower errors than standard sparse attention, while the hybrid-order variant yields further improvements. This validates the numerical precision of our method, confirming its ability to minimize output error.

\begin{figure}[!t]
    \centering
    \def\pisaAblationGap{1.2pt}
    \def\pisaAblationGroupGap{8pt}
    \def\pisaAblationHeadH{9pt}
    \pgfmathsetlengthmacro{\pisaAblationTile}{(\linewidth-6*\pisaAblationGap-2*\pisaAblationGroupGap)/9}
    \input{fig/attn_map_abla}
    \vspace{-28pt}
    \caption{Approximation ablation (85\% sparsity): FLUX.1 images (left, middle) and attention maps (right). \ourmethod{}-0th loses fine details and attention structure, which the hybrid correction recovers.}
    \label{fig:abla_qualitative}
\end{figure}

\vspace{-4pt}
\paragraph{Effect of Approximation Strategy.} Beyond output error of attention, we evaluate generation similarity against the \textit{exact block-wise first-order} variant (Eq.~\eqref{eq:block-1st-d}). As shown in Table~\ref{tab:ablation1}, our hybrid approximation significantly improves similarity over the zeroth-order method and approaches the exact baseline with superior efficiency, demonstrating an optimal fidelity-efficiency trade-off.

\vspace{-4pt}
\paragraph{Qualitative Comparison.} We visualized the approximation effects on FLUX.1 in Figure~\ref{fig:abla_qualitative}. While the 0th-order expansion maintains semantics and structural integrity, it struggles with local details. In contrast, the hybrid strategy significantly recovers these fine-grained details. This provides intuitive visual evidence that hybrid strategy effectively preserves human-perceptible details.

\vspace{-4pt}
\section{Conclusion}
\vspace{-4pt}
In this work, we propose a unified ``exact-or-approximate'' framework that goes beyond the standard ``keep-or-drop'' sparse attention paradigm and achieves a better efficiency-accuracy trade-off for visual generation. Our current scope focuses on multi-step diffusion models with dense warmup. Future work will explore the co-design of \ourmethod{} with few-step diffusion~\cite{yin2024onestep,yin2024improved,shao2026fastlightgen,bai2026optimizing}, advanced sampling~\cite{bai2025zigzag,bai2026weak}, and feature caching~\cite{liu2025teacache,liu2025taylorseer}, as well as its application to large language models (LLMs).

\clearpage
\setlength{\floatsep}{\pisasavedfloatsep}
\setlength{\textfloatsep}{\pisasavedtextfloatsep}
\renewcommand{\acksection}{\section*{Acknowledgments}}
\begin{ack}
We thank the Red Bird MPhil Program at the Hong Kong University of Science and Technology (Guangzhou) for providing support, computing resources, and funding, as well as Red Bird MPhil project supervisor Dr.\ Jingwei Yang. This work was also partially supported by the Hong Kong University of Science and Technology (Guangzhou) Kunpeng \& Ascend Center of Cultivation.
\end{ack}

\bibliographystyle{unsrtnat}
{\small
\bibliography{references}
}

\clearpage
\appendix
\label{appendix}

\section{Expanded Related Work}
\paragraph{Sparse Attention for DiTs.} Sparse attention reduces computational complexity by selectively computing only critical key-value pairs. In video generation, static sparse attention~\cite{xisparse,zhang2025fast} relies on attention masks derived from intrinsic spatio-temporal sparsity priors, whereas dynamic sparse attention~\cite{yang2025sparse,xia2025training,zhang2025spargeattn,sun2025vorta,zhang2025training} determines which key-value blocks to compute at runtime. To maximize hardware utilization, these methods typically operate at block granularity, where a block of queries attends to a shared set of key-value blocks and the remaining blocks are discarded. Some LLM methods achieve finer granularity by letting token-wise queries attend to block-wise key-values~\cite{lu2025moba,yuan2025native}, but they rely on the properties of causal decoding and are difficult to transfer to bidirectional attention. In all these methods, the discarded blocks contribute nothing to the output, so the attention distribution deviates increasingly from the dense one as sparsity grows.

\vspace{-2pt}
\paragraph{Bidirectional Linear Attention.} Linear attention reduces the complexity to linear by decomposing $\exp(\boldsymbol{q}\boldsymbol{k}^\top)$ into $\phi(\boldsymbol{q})\phi(\boldsymbol{k})^\top$ and computing $\boldsymbol{k}^\top\boldsymbol{v}$ first. Causal LLMs have dropped the non-negativity constraint on $\phi$ by post-normalizing the output~\cite{qin2022devil,sun2023retentive}, whereas bidirectional variants for vision still adopt non-negative kernels (e.g., $\text{ELU}(\cdot)+1$) to mimic the normalization of softmax~\cite{han2023flatten,liu2024linfusion,xie2024sana,mengpolaformer}. In either case, the kernel only approximates softmax attention, so these models must be trained or fine-tuned and cannot directly reuse the weights of a pre-trained softmax model. In contrast, \ourmethod{} keeps the exact softmax kernel and only approximates the contribution of unselected blocks, which allows it to be applied to pre-trained DiTs without training.

\vspace{-2pt}
\paragraph{Linear Attention with Taylor Expansion.} In LLMs, several studies~\cite{arora2024simple,gelada2025scaling} linearize softmax attention via Taylor expansion. \textit{Based}~\cite{arora2024simple} and \textit{power attention}~\cite{gelada2025scaling} compute second-order terms exactly with the feature map $\phi_p(q)\phi_p(k)=(q^\top k)^p$, which, however, significantly enlarges the head dimension: with a base head dimension of 64, \textit{Based} requires $64\to4096$, and even the symmetric powers map of \textit{power attention} requires $64\to2080$. The resulting $\mathcal{O}(Nd^2)$ cost is tolerable only in token-wise causal decoding. Moreover, since these methods expand around 0 for the entire sequence, they must be trained from scratch. \ourmethod{} instead expands each unselected block around its own centroid and keeps only the first-order term, which keeps the approximation accurate for pre-trained attention without enlarging the head dimension.

\vspace{-2pt}
\paragraph{Hybrid Attention.} Hybrid architectures combine softmax and linear attention either across or within layers, typically by computing both branches over the full sequence and taking a weighted sum, or by assigning attention heads to each mechanism. More recent efforts hybridize at a finer granularity: NHA~\cite{du2025native} compresses tokens outside a sliding window into linear-RNN key-value slots and applies a single softmax over these slots and the window tokens, and SLA~\cite{zhang2025sla} combines sparse attention with linear attention over a subset of key-value pairs. However, SLA aggregates the two branches by direct summation, which breaks the normalized form of softmax attention, and both introduce new components that must be trained, precluding training-free reuse of pre-trained weights.

\section{Compare with Other Methods}
In contrast to recent studies, our method differs from existing approaches in three key aspects. 

First, unlike pure sparse attention~\cite{xuxattention,xisparse,yang2025sparse,zhang2025spargeattn} which adheres to a ``keep-or-drop'' paradigm that discards the majority of key-value pairs, our method adopts an ``exact-or-approximate'' paradigm, blending sparse exact computation with efficient approximation.

Second, in contrast to hybrid attention~\cite{du2025native,zhang2025sla} that either sum softmax and linear attention outputs or require training new memory components, we propose a mathematically rigorous mixing strategy. By employing block-wise Taylor expansion to approximate both the numerator and denominator, we unify the hybridization of softmax and linear attention within the canonical normalization framework.

Third, compared to other Taylor-based methods~\cite{arora2024simple,gelada2025scaling} that perform global expansion over the entire sequence, resulting in prohibitive approximation errors that prevent the reuse of pre-trained weights, our method performs expansion around the local mean of each block. This significantly reduces approximation error, enabling the direct, training-free inheritance of pre-trained weights.

To the best of our knowledge, we are the first to natively integrate block-wise Taylor expansion and sparse attention into a unified framework.

\clearpage
\section{Error Analysis of PISA}
\label{appendix:error_bound}

\paragraph{Notation.}
We use the scale-omitted notation of Sec.~\ref{sec:method}; equivalently, one can
replace $\boldsymbol q_t$ below by $\boldsymbol q_t/\sqrt d$. For a fixed query row
$\boldsymbol q_t$, define
\[
\bar{\boldsymbol k}_j:=\frac1B\sum_{n=1}^B\boldsymbol k_{j,n},\quad
\alpha_{t,j}:=\exp(\boldsymbol q_t\bar{\boldsymbol k}_j^\top),\quad
\delta_{t,j,n}:=\boldsymbol q_t(\boldsymbol k_{j,n}-\bar{\boldsymbol k}_j)^\top .
\]
Then $\sum_{n=1}^B\delta_{t,j,n}=0$. Let
\[
\boldsymbol H_j:=\sum_{n=1}^B
(\boldsymbol k_{j,n}-\bar{\boldsymbol k}_j)^\top\boldsymbol v_{j,n},
\qquad
\bar{\boldsymbol H}:=\frac1N\sum_{j=1}^N\boldsymbol H_j .
\]
Let $\mathcal S_i$ and $\mathcal U_i$ be the selected and unselected block sets. For compactness,
define the exact selected-block denominator and numerator as
\[
\mathcal D_t^{(\mathcal S)}
:=
\sum_{j\in\mathcal S_i}\sum_{n=1}^B
\exp(\boldsymbol q_t\boldsymbol k_{j,n}^\top),
\qquad
\mathcal N_t^{(\mathcal S)}
:=
\sum_{j\in\mathcal S_i}\sum_{n=1}^B
\exp(\boldsymbol q_t\boldsymbol k_{j,n}^\top)\boldsymbol v_{j,n}.
\]
The dense denominator and numerator are
\[
\mathcal D_t^\star
=
\mathcal D_t^{(\mathcal S)}
+
\sum_{j\in\mathcal U_i}
\alpha_{t,j}\sum_{n=1}^B\exp(\delta_{t,j,n}),
\]
\[
\mathcal N_t^\star
=
\mathcal N_t^{(\mathcal S)}
+
\sum_{j\in\mathcal U_i}
\alpha_{t,j}\sum_{n=1}^B
\exp(\delta_{t,j,n})\boldsymbol v_{j,n}.
\]
Thus $\boldsymbol o_t^\star=\mathcal N_t^\star/\mathcal D_t^\star$. PISA computes
\[
\widehat{\mathcal D}_t
=
\mathcal D_t^{(\mathcal S)}
+
B\sum_{j\in\mathcal U_i}\alpha_{t,j},
\]
\[
\widehat{\mathcal N}_t
=
\mathcal N_t^{(\mathcal S)}
+
\sum_{j\in\mathcal U_i}
\alpha_{t,j}\sum_{n=1}^B\boldsymbol v_{j,n}
+
\boldsymbol q_t\bar{\boldsymbol H}
\sum_{j\in\mathcal U_i}\alpha_{t,j},
\]
and $\widehat{\boldsymbol o}_t=\widehat{\mathcal N}_t/\widehat{\mathcal D}_t$.

Finally, define
\[
S_{2,t}:=
\sum_{j\in\mathcal U_i}
\alpha_{t,j}\sum_{n=1}^B\delta_{t,j,n}^2,
\quad
\eta_t:=\max_{j\in\mathcal U_i,n}|\delta_{t,j,n}|,
\]
\[
\boldsymbol R_t:=
\sum_{j\in\mathcal U_i}
\alpha_{t,j}(\boldsymbol H_j-\bar{\boldsymbol H}),
\quad
A_t:=\sum_{j\in\mathcal U_i}\alpha_{t,j}.
\]

\begin{theorem}[Sharp Total Error of PISA]
\label{thm:appendix_sharp_total}
Assume $\|\boldsymbol v_{j,n}\|_2\le C_v$. Then
\[
\left\|
\widehat{\boldsymbol o}_t-\boldsymbol o_t^\star
\right\|_2
\le
\frac{
C_v e^{\eta_t}S_{2,t}
+
\|\boldsymbol q_t\boldsymbol R_t\|_2
}{
\widehat{\mathcal D}_t
}.
\]
Moreover, if $\|\boldsymbol q_t\|_2\le C_q$ and
$M_t:=\max_{j\in\mathcal U_i}\|\boldsymbol H_j-\bar{\boldsymbol H}\|_2$, then
\[
\left\|
\widehat{\boldsymbol o}_t-\boldsymbol o_t^\star
\right\|_2
\le
\frac{
C_v e^{\eta_t}S_{2,t}
+
C_qM_tA_t
}{
\widehat{\mathcal D}_t
}.
\]
\end{theorem}

\begin{proof}
The proof follows the computation path of PISA: first approximate the unselected
tail blocks by Taylor expansion, then replace the block-wise first-order matrices by
the global statistic.

For any $|x|\le\eta_t$, Taylor's theorem gives
\[
0\le e^x-1-x\le \frac12 e^{\eta_t}x^2 .
\]
Since $\sum_{n=1}^B\delta_{t,j,n}=0$, the denominator error is
\[
\mathcal D_t^\star-\widehat{\mathcal D}_t
=
\sum_{j\in\mathcal U_i}\alpha_{t,j}
\sum_{n=1}^B
\left(e^{\delta_{t,j,n}}-1-\delta_{t,j,n}\right),
\]
and therefore
\[
0\le
\mathcal D_t^\star-\widehat{\mathcal D}_t
\le
\frac12 e^{\eta_t}S_{2,t}.
\]
Similarly, define the exact block-wise first-order numerator
\[
\mathcal N_t^{(1)}
:=
\mathcal N_t^{(\mathcal S)}
+
\sum_{j\in\mathcal U_i}
\alpha_{t,j}\sum_{n=1}^B\boldsymbol v_{j,n}
+
\boldsymbol q_t
\sum_{j\in\mathcal U_i}\alpha_{t,j}\boldsymbol H_j .
\]
Then
\[
\mathcal N_t^\star-\mathcal N_t^{(1)}
=
\sum_{j\in\mathcal U_i}\alpha_{t,j}
\sum_{n=1}^B
\left(e^{\delta_{t,j,n}}-1-\delta_{t,j,n}\right)\boldsymbol v_{j,n},
\]
which implies
\[
\|\mathcal N_t^\star-\mathcal N_t^{(1)}\|_2
\le
\frac12 C_v e^{\eta_t}S_{2,t}.
\]

Because dense attention is a convex combination of value vectors,
$\|\mathcal N_t^\star/\mathcal D_t^\star\|_2\le C_v$. Hence
\[
\left\|
\frac{\mathcal N_t^\star}{\mathcal D_t^\star}
-
\frac{\mathcal N_t^{(1)}}{\widehat{\mathcal D}_t}
\right\|_2
\le
\frac{
\|\mathcal N_t^\star-\mathcal N_t^{(1)}\|_2
+
C_v|\mathcal D_t^\star-\widehat{\mathcal D}_t|
}{
\widehat{\mathcal D}_t
}
\le
\frac{C_v e^{\eta_t}S_{2,t}}{\widehat{\mathcal D}_t}.
\]

It remains to bound the global first-order replacement. Since
\[
\mathcal N_t^{(1)}-\widehat{\mathcal N}_t
=
\boldsymbol q_t
\sum_{j\in\mathcal U_i}\alpha_{t,j}
(\boldsymbol H_j-\bar{\boldsymbol H})
=
\boldsymbol q_t\boldsymbol R_t,
\]
we have
\[
\left\|
\frac{\mathcal N_t^{(1)}}{\widehat{\mathcal D}_t}
-
\frac{\widehat{\mathcal N}_t}{\widehat{\mathcal D}_t}
\right\|_2
=
\frac{\|\boldsymbol q_t\boldsymbol R_t\|_2}{\widehat{\mathcal D}_t}.
\]
Combining the Taylor error and the global replacement error proves the first claim.

For the second claim,
\[
\|\boldsymbol R_t\|_2
\le
\sum_{j\in\mathcal U_i}\alpha_{t,j}
\|\boldsymbol H_j-\bar{\boldsymbol H}\|_2
\le
M_tA_t,
\]
so
\[
\|\boldsymbol q_t\boldsymbol R_t\|_2
\le
\|\boldsymbol q_t\|_2\|\boldsymbol R_t\|_2
\le
C_qM_tA_t.
\]
\end{proof}

\begin{corollary}[Tail-Mass Form]
\label{cor:appendix_tail_mass}
Let
\[
\tau_t:=
\sum_{j\in\mathcal U_i}\sum_{n=1}^B
\exp(\boldsymbol q_t\boldsymbol k_{j,n}^\top),
\quad
\rho_t:=\frac{\tau_t}{\mathcal D_t^\star},
\quad
\zeta_t:=1-\frac{\widehat{\mathcal D}_t}{\mathcal D_t^\star}.
\]
Then $0\le\zeta_t<1$, and under the assumptions of
Theorem~\ref{thm:appendix_sharp_total},
\[
\left\|
\widehat{\boldsymbol o}_t-\boldsymbol o_t^\star
\right\|_2
\le
\frac{\rho_t}{1-\zeta_t}
\left(
C_v e^{\eta_t}\eta_t^2
+
\frac{C_qM_t}{B}
\right).
\]
\end{corollary}

\begin{proof}
By Jensen's inequality,
\[
B\alpha_{t,j}
=
B\exp\left(
\frac1B\sum_{n=1}^B
\boldsymbol q_t\boldsymbol k_{j,n}^\top
\right)
\le
\sum_{n=1}^B
\exp(\boldsymbol q_t\boldsymbol k_{j,n}^\top).
\]
Thus $BA_t\le\tau_t$ and $A_t\le\tau_t/B$. Also,
\[
S_{2,t}
\le
\eta_t^2\sum_{j\in\mathcal U_i}B\alpha_{t,j}
\le
\eta_t^2\tau_t.
\]
The same Jensen inequality gives
$0<\widehat{\mathcal D}_t\le\mathcal D_t^\star$, hence
$0\le\zeta_t<1$ and
$\widehat{\mathcal D}_t=(1-\zeta_t)\mathcal D_t^\star$.
Substituting these relations into Theorem~\ref{thm:appendix_sharp_total} yields
the stated bound.
\end{proof}

\section{Covariance-Aware Block Selection}
\label{appendix:topk}
Recalling the theoretical error bound in Theorem~\ref{thm:appendix_sharp_total} that $\text{error} \propto \exp(\boldsymbol{q}\boldsymbol{k}^\top)\cdot \|\boldsymbol{H}\|_2$, our goal is to identify blocks with both large attention scores and high approximation errors. We simplify this objective by rectifying the attention score using $M_j := \|\boldsymbol{H}_j - \bar{\boldsymbol{H}}\|$. To eliminate the influence of the absolute magnitude of $M_j$, we define the block-wise routing factor as 
\( \exp(\boldsymbol{q}_t \bar{\boldsymbol{k}}_{j}^\top) \cdot {M_j}/{\bar{M}} \), where $\bar{M}$ is the mean of $M_{j\in\mathcal{U}}$. To integrate this into the softmax operation, we take its logarithm:
\begin{equation}\label{eq:appendix_topk}
    \text{Score}_{t,j} = \text{Softmax}\left( \frac{\boldsymbol{q}_t \bar{\boldsymbol{k}}_j^\top}{\sqrt{d}} + \log(M_j) - \log(\bar{M}) \right).
\end{equation}
Since $\bar{M}$ is a constant term, it cancels out during the Softmax normalization, thereby yielding Eq.~\eqref{eq:topk}.

\section{Implementation Details for Reproducibility}
\label{appendix:detail}
We detail the specific configuration of the warmup strategy in Table~\ref{tab:config_video} and Table~\ref{tab:config_image}. Beyond this, to ensure a strictly fair comparison, all other parameters regarding sparse attention adhere to the official implementations. When integrating different attention methods into video and image generation models, we utilize the official recommended configurations for all remaining hyperparameters (e.g., sampling steps, classifier-free guidance (CFG) scale), without further enumeration.

\begin{table*}[!ht]
\centering
\scriptsize
\label{tab:reproduce_combined}
\begin{minipage}[t]{0.44\linewidth}
    \centering
    \caption{Configuration of Video Generation} 
    \vspace{-4pt}
    \label{tab:config_video}
    \begin{tabularx}{\linewidth}{l l | >{\centering\arraybackslash}X | >{\centering\arraybackslash}X | >{\centering\arraybackslash}X}
    \toprule
    \multirow{2}{*}{\textbf{Ref.}} & \multirow{2}{*}{\textbf{Config}} & \textbf{Wan2.1} & \textbf{Wan2.1} & \textbf{Hunyuan} \\
     & & \textbf{1.3B} & \textbf{14B} & \textbf{13B} \\
    \midrule
    \multirow{2}{*}{\textit{Table~\ref{tab:main_video}}}
    & Dense Layer & 1 & 1 & 1 \\
    & Dense Step  & 15 & 10 & 10 \\
    \midrule
    \multirow{2}{*}{\textit{Table~\ref{tab:without_warmup}}}
    & Dense Layer & 0 & 0 & 0 \\
    & Dense Step  & 0 & 0 & 0 \\
    \bottomrule
    \end{tabularx}
\end{minipage}
\hfill
\begin{minipage}[t]{0.54\linewidth}
    \centering
    \caption{Configuration of Image Generation}
    \vspace{-4pt}
    \label{tab:config_image}
    \begin{tabularx}{\linewidth}{l l | >{\centering\arraybackslash}X | >{\centering\arraybackslash}X | >{\centering\arraybackslash}X}
    \toprule
    \multirow{2}{*}{\textbf{Ref.}} & \multirow{2}{*}{\textbf{Config}} & \textbf{SD 3.5 M} & \textbf{FLUX.1} & \textbf{FLUX.1} \\
     & & \textit{1024p} & \textit{2048p} & \textit{1024p}\\
    \midrule
    \multirow{2}{*}{\textit{Table~\ref{tab:main_img}}}
    & Dense Layer & 4 & 4 & 4 \\
    & Dense Step  & 0 & 0 & 0 \\
    \midrule
    \multirow{2}{*}{\textit{Table~\ref{tab:ablation2}}}
    & Dense Layer & - & - & 4 \\
    & Dense Step  & - & - & 0 \\
    \bottomrule
    \end{tabularx}
\end{minipage}

\vspace{-4pt}
\end{table*}

\section{More Experiments}
\label{appendix:exp}

\begin{table}[t]
\scriptsize
    \centering
    \caption{Comparison of different sparse attention on video generation without warmup strategy.}
    \vspace{7pt}
    \label{tab:without_warmup}
        \begin{tabularx}{0.99\linewidth}{l l C|CC|CC|C}
            \toprule
            \multirow{2}{*}{\textbf{Model}} 
            & \multirow{2}{*}{\textbf{Method}} 
            & \multirow{2}{*}{\textbf{Sparsity}} 
            & \multicolumn{2}{c|}{\textbf{VBench(\%) $\uparrow$}} 
            & \multicolumn{2}{c|}{\textbf{Similarity}} 
            & \multirow{2}{*}{\textbf{Speedup$\uparrow$}} \\
            \cmidrule(lr){4-5} \cmidrule(lr){6-7}
            & & 
            & \textbf{I.Q.} 
            & \textbf{A.Q.} 
            & \textbf{PSNR$\uparrow$} 
            & \textbf{LPIPS$\downarrow$} 
            & \\
            \midrule
            
            \multirow{3}{*}{\textit{Wan2.1-1.3B}}
            & Sparge & 87.5\% & 60.10 & 55.26 & 10.92 & 0.397 & 2.72$\times$ \\[0.6pt]
            & SVG2 & 87.1\% & 63.22 & 57.25 & 13.42 & 0.292 & 2.33$\times$ \\[0.6pt]
            & \cellcolor{gray!12.5}Ours & \cellcolor{gray!12.5}87.5\% & \cellcolor{gray!12.5}66.08 & \cellcolor{gray!12.5}60.32 & \cellcolor{gray!12.5}14.16 & \cellcolor{gray!12.5}0.267 & \cellcolor{gray!12.5}3.06$\times$ \\
            \midrule
            
            \multirow{3}{*}{\textit{Wan2.1-14B}}
            & Sparge & 87.5\% & 60.03 & 51.45 & \hspace{\widthof{0}}9.48 & 0.520 & 2.59$\times$ \\[0.6pt]
            & SVG2 & 87.1\% & 58.85 & 55.72 & 10.67 & 0.489 & 2.37$\times$ \\[0.6pt]
            & \cellcolor{gray!12.5}Ours & \cellcolor{gray!12.5}87.5\% & \cellcolor{gray!12.5}69.95 & \cellcolor{gray!12.5}64.47 & \cellcolor{gray!12.5}12.04 & \cellcolor{gray!12.5}0.398 & \cellcolor{gray!12.5}2.75$\times$ \\
            \midrule
            
            \multirow{3}{*}{\textit{Hunyuan-13B}}
            & Sparge & 87.5\% & 65.10 & 60.34 & 14.91 & 0.283 & 3.65$\times$ \\[0.6pt]
            & SVG2 & 87.1\% & 63.96 & 60.34 & 16.20 & 0.275 & 3.60$\times$ \\[0.6pt]
            & \cellcolor{gray!12.5}Ours & \cellcolor{gray!12.5}87.6\% & \cellcolor{gray!12.5}67.61 & \cellcolor{gray!12.5}61.65 & \cellcolor{gray!12.5}18.73 & \cellcolor{gray!12.5}0.264 & \cellcolor{gray!12.5}3.82$\times$ \\
            
            \bottomrule
        \end{tabularx}
\end{table}

\begin{figure}[t]
    \centering
    \begin{minipage}{0.495\linewidth}
        \centering
        \input{fig/attn_map_appendix1}
    \end{minipage}
    \hfill
    \begin{minipage}{0.495\linewidth}
        \centering
        \input{fig/attn_map_appendix2}
    \end{minipage}
    \caption{Visualization of attention patterns from uncurated random slices on Wan2.1-1.3B/14B.}
    \label{fig:attn_maps_appendix}
\end{figure}

\paragraph{Generation without Warmup.} We provide video generation results without a warmup strategy. As shown in Table~\ref{tab:without_warmup}, our method consistently outperforms the baselines. 

\paragraph{Attention Patterns.} We also include additional attention pattern comparisons across different models, layers, inference steps, and attention heads. For clarity and conciseness, we visualize representative cropped regions, as shown in Figure~\ref{fig:attn_maps_appendix}.

\begin{table}[t]
\centering
\footnotesize
\setlength{\tabcolsep}{4pt}
\renewcommand{\arraystretch}{0.9}
\setlength{\aboverulesep}{0.8pt}
\setlength{\belowrulesep}{1pt}
\caption{Text-to-image comparison on FLUX.1 at 2048$\times$2048 resolution with warmup strategy.}
\label{tab:main_img_2048}
\vspace{7pt}
\begin{tabularx}{\textwidth}{l c *{7}{>{\centering\arraybackslash}X} c}
\toprule
\multirow{2}{*}{\textbf{Method}} & \multirow{2}{*}{\textbf{Sparsity}} & \multicolumn{4}{c}{\textbf{Quality}} & \multicolumn{3}{c}{\textbf{Similarity}} & \multicolumn{1}{c}{\textbf{Efficiency}} \\
\cmidrule(lr){3-6}\cmidrule(lr){7-9}\cmidrule(lr){10-10}
 & & \textbf{FID}$\downarrow$ & \textbf{IR}$\uparrow$ & \textbf{MPS}$\uparrow$ & \textbf{HPSv2}$\uparrow$ & \textbf{SSIM}$\uparrow$ & \textbf{PSNR}$\uparrow$ & \textbf{LPIPS}$\downarrow$ & \textbf{Latency}$\downarrow$ \\
\midrule
Dense-Attn & \hspace{\widthof{0}}0\% & 17.57 & 0.940 & -- & 30.26 & -- & -- & -- & 48.68 s\\
  MInference & 90\% & 19.66 & 0.910 & 44.33\% & 29.03 & 0.669 & 18.38 & 0.164 & 35.74 s \\
  SpargeAttn & 90\% & 19.72 & 0.908 & 46.19\% & 28.69 & 0.655 & 18.11 & 0.170 & 32.90 s \\
  \cellcolor{gray!12.5}Ours & \cellcolor{gray!12.5}90\% & \cellcolor{gray!12.5}17.11 & \cellcolor{gray!12.5}0.932 & \cellcolor{gray!12.5}50.44\% & \cellcolor{gray!12.5}30.99 & \cellcolor{gray!12.5}0.714 & \cellcolor{gray!12.5}20.72 & \cellcolor{gray!12.5}0.137 & \cellcolor{gray!12.5}30.04 s \\
\bottomrule
\end{tabularx}
\end{table}

\paragraph{High-Resolution Image Generation.} Table~\ref{tab:main_img_2048} further reports FLUX.1 at 2048$\times$2048 resolution, where \ourmethod{} again outperforms the baselines at 90\% sparsity.

\paragraph{Controlled Design Ablations.} Table~\ref{tab:rebuttal_design} isolates the two design choices of \ourmethod{}. With routing fixed (left), the average-slope coefficient $\beta_t$ lowers the relative $\ell_2$ error from 0.1455 (centroid $\beta_t$) to 0.1440, close to the 0.1420 of blockwise $H$, while all first-order variants improve over zeroth-order approximation (0.1533). With the approximation and budget fixed (right), combining attention and covariance cues in routing outperforms either cue alone on all three metrics.

\begin{table}[H]
\centering
\scriptsize
\setlength{\tabcolsep}{2.5pt}
\renewcommand{\arraystretch}{0.9}
\setlength{\aboverulesep}{0.8pt}
\setlength{\belowrulesep}{1pt}
\caption{\textbf{Controlled design ablations}. Left: fixed routing; right: fixed approximation and budget.}
\label{tab:rebuttal_design}
\vspace{7pt}
\begin{minipage}[t]{0.42\linewidth}
\centering
\begin{tabularx}{\linewidth}[t]{lY}
\toprule
\textbf{Approximation} & \textbf{Rel. $\ell_2$}$\downarrow$ \\
\midrule
Zeroth order & 0.1533 \\
Centroid $\beta_t$ + global $H$ & 0.1455 \\
\rowcolor{gray!12.5} Average $\beta_t$ + global $H$ & 0.1440 \\
Average $\beta_t$ + blockwise $H$ & 0.1420 \\
\bottomrule
\end{tabularx}
\end{minipage}\hfill
\begin{minipage}[t]{0.54\linewidth}
\centering
\begin{tabularx}{\linewidth}[t]{lYYY}
\toprule
\textbf{Router} & \textbf{PSNR}$\uparrow$ & \textbf{SSIM}$\uparrow$ & \textbf{LPIPS}$\downarrow$ \\
\midrule
Random & 11.815 & 0.2332 & 0.8733 \\
Covariance-only & 15.871 & 0.5905 & 0.4277 \\
Attention-only & 16.572 & 0.6038 & 0.3964 \\
\rowcolor{gray!12.5} Combined & 17.266 & 0.6290 & 0.3673 \\
\bottomrule
\end{tabularx}
\end{minipage}
\end{table}

\paragraph{Effect of Block Size.}
We study the impact of block size on quality and efficiency. 
Table~\ref{tab:ablation2} shows that 64$\times$64 achieves the best trade-off. 
Larger blocks may hurt efficiency due to register spilling and degrade quality due to coarser granularity. 
Group size only affects efficiency, as it controls the number of block centroids loaded per approximation step and is autotuned at kernel compile time.

\begin{table}[H]
\centering
\scriptsize
\setlength{\tabcolsep}{6pt}
\renewcommand{\arraystretch}{0.95}
\caption{\textbf{Ablations of block size}. Block size refers to the Q block and KV block sizes.}
\label{tab:ablation2}
\vspace{7pt}
\begin{minipage}[t]{0.485\linewidth}
\centering
\begin{tabularx}{\linewidth}[t]{lYYY}
\toprule
\multicolumn{4}{c}{\textit{\textbf{FLUX.1-dev}}} \\
\midrule
\textbf{Block Size} & \textbf{SSIM}$\uparrow$ & \textbf{PSNR}$\uparrow$ & \textbf{Lat.}$\downarrow$ \\
\midrule
128$\times$128 & 0.606 & 12.11 & 7.12 s \\
128$\times$\phantom{0}64 & 0.661 & 15.99 & 6.66 s \\
\phantom{0}64$\times$128 & 0.648 & 15.09 & 7.08 s \\
\cellcolor{gray!12.5}\phantom{0}64$\times$\phantom{0}64 & \cellcolor{gray!12.5}0.682 & \cellcolor{gray!12.5}17.07 & \cellcolor{gray!12.5}6.87 s \\
\bottomrule
\end{tabularx}
\end{minipage}\hfill
\begin{minipage}[t]{0.485\linewidth}
\centering
\begin{tabularx}{\linewidth}[t]{lYYY}
\toprule
\multicolumn{4}{c}{\textit{\textbf{Wan2.1-1.3B}}} \\
\midrule
\textbf{Block Size} & \textbf{SSIM}$\uparrow$ & \textbf{PSNR}$\uparrow$ & \textbf{Lat.}$\downarrow$ \\
\midrule
128$\times$128 & 0.785 & 21.39 & 57.6 s \\
128$\times$\phantom{0}64 & 0.796 & 22.27 & 48.4 s \\
\phantom{0}64$\times$128 & 0.792 & 22.11 & 54.5 s \\
\cellcolor{gray!12.5}\phantom{0}64$\times$\phantom{0}64 & \cellcolor{gray!12.5}0.800 & \cellcolor{gray!12.5}22.62 & \cellcolor{gray!12.5}51.0 s \\
\bottomrule
\end{tabularx}
\end{minipage}
\end{table}

\section{More Generation Samples}
\label{appendix:samples}
We provide additional image and video generation samples in Figure~\ref{fig:appendix_imgs} and Figure~\ref{fig:appendix_vids}.

\begin{figure}[t]
    \centering
    \input{fig/image_samples}
    \caption{Text-to-Image generation samples on FLUX.1-dev. The SpargeAttn under the 80\% sparsity while \ourmethod{} under the 85\% sparsity.}
    \label{fig:appendix_imgs}
\end{figure}

\begin{figure}[t]
    \centering
    \input{fig/video_samples}
    \caption{Text-to-Video generation samples on Wan2.1-14B. The \ourmethod{} under the 87.5\% sparsity with 10 steps and 1 layer warmup.}
    \label{fig:appendix_vids}
\end{figure}

\end{document}